\documentclass[acmlarge]{acmart}
\usepackage{algorithm}
\usepackage{algpseudocode}
\usepackage{amsmath}
\usepackage{amsthm}
\usepackage{graphicx}
\usepackage{subcaption}

\AtBeginDocument{%
  }

\begin{document}

\title{eFedLLM: Efficient LLM Inference Based on Federated Learning}

\author{Shengwen Ding}
\affiliation{%
  \institution{Independent Researcher}
  \city{Pittsburgh, PA}
  \country{USA}}
\email{jenny.dingsw@gmail.com}

\author{Chenhui Hu}
\affiliation{%
  \institution{Independent Researcher}
  \city{Cambridge, MA}
  \country{USA}}

\renewcommand{\shortauthors}{Ding et al.}

\begin{abstract}
  Large Language Models (LLMs) herald a transformative era in artificial intelligence (AI). However, the expansive scale of data and parameters of LLMs requires high-demand computational and memory resources, restricting their accessibility to a broader range of users and researchers. This paper introduces an effective approach that enhances the operational efficiency and affordability of LLM inference. By utilizing transformer-based federated learning (FL) with model-parallel distributed training, our model efficiently distributes the computational loads and memory requirements across a network of participants. This strategy permits users, especially those with limited resources to train state-of-the-art LLMs collaboratively. We also innovate an incentive mechanism within the FL framework, rewarding constructive contributions and filtering out malicious activities, thereby safeguarding the integrity and reliability of the training process. Concurrently, we leverage memory hierarchy strategies and Singular Value Decomposition (SVD) on weight matrices to boost computational and memory efficiencies further. Our results, derived from formulaic analyses and numerical calculations, demonstrate significant optimization of resource use and democratize access to cutting-edge LLMs, ensuring that a wide scale of users can both contribute to and benefit from these advanced models.   
  
\end{abstract}





\maketitle

%%%%%%%%%%%%%%%%%%%% Introduction %%%%%%%%%%%%%%%%%%%%
\section{Introduction}
Large language models (LLMs) represent a significant advancement in artificial intelligence (AI), specially designed to process and produce natural language text. The substantial improvement of LLMs over traditional machine learning (ML) models can be attributed to their large-scale training datasets collected from multiple sources (e.g. books, articles, web pages, and images), expansive scale of parameters, and enhanced precision in floating-point computations. These models typically utilize Transformer architectures \cite{vaswani2017attention}, which are pivotal for their ability to recognize patterns and rules of language, allowing them to predict and generate appropriate responses. One unique attribute of the Transformer is the Attention mechanism \cite{vaswani2017attention}, which dynamically weighs the importance of different tokens (i.e. the text inputs) in a sequence, thereby improving the model’s ability to focus on relevant parts of the input when predicting the most probable output. LLMs power a variety of transformative generative AI applications, such as ChatGPT \cite{brown2020language}, Phi-3 \cite{abdin2024phi}, StyleGAN \cite{karras2020analyzing}, and BERT \cite{devlin2018bert}, dramatically altering many aspects of human experience. 

Despite their advanced capabilities, LLMs are still inaccessible to many users and researchers who do not have enough computational resources, leading to challenges in model training and inference \cite{conway2024opML, sun2024zkLLM, cai2024medusa}. With the rapid development of LLM architectures, the models have grown significantly in size and complexity. For example, BERT-Large, introduced in 2018, comprises 345 million parameters \cite{LambdaLabs2020}, requiring substantial computational resources for training and inference. By 2020, GPT-3 expanded this scale to 175 billion parameters \cite{brown2020language, LambdaLabs2020}, demanding an estimated 700 GB of memory just to store these parameters, making it impractical for single GPU systems. GPT-4 is even estimated around 1.4 trillion parameters, further amplifying these challenges. The GPU memory requirements associated with such models make them inaccessible to users who cannot afford the high-end hardware necessary for such computations,
which might restrict the democratization of AI technology. 

Addressing the high-demand computational and memory requirements of LLMs, this paper investigates effective strategies to enhance operational efficiency and accessibility for a broad range of users and researchers. In this work, cost-effective methods are explored for running pre-trained LLMs primarily during inference, their most common application. The aim is to enhance the operational efficiency and affordability of LLM inference by integrating them with transformer-based federated learning (FL) \cite{li2020review, zhang2021survey} and targeted algorithmic optimizations. Specifically, the contributions of the paper are threefold: 
\begin{itemize}
    \item \textbf{Distributed Training Architecture}: FL is utilized to distribute the training of different layers of the Transformer models across multiple decentralized devices or servers. This method not only optimizes the training process for greater efficiency but also enhances security. Within this FL framework, users equipped with varying levels of computational resources collaboratively train LLMs in a sequence, with the output from one server feeding directly into the next, thereby enabling efficient use of distributed resources.
    \item \textbf{Incentive Mechanism}: An incentive mechanism is introduced within the FL framework that rewards constructive contributions and deters malicious activities. This innovation preserves the integrity and reliability of the distributed training process, fostering a secure training environment. 
    \item \textbf{Theoretical Optimization}: Theoretical methods such as memory hierarchy strategies and Singular Value Decomposition (SVD) on weight matrices are implemented. These enhancements significantly improve computational and memory efficiencies during the inference phase, reducing operational costs and enabling more widespread use of advanced LLMs. Base on the analysis of one layer of the BERT model, our results reveal that the bandwidth reduce rate decreases as the compression ratio of the weight matrices increases. When compressing the model's weight matrices to 70\% of their original size,  the bandwidth usage can be reduced by about 60\%. 
\end{itemize}

The rest of this paper is organized as such: Section~\ref{sec:Background} provides a background of FL and Transformer models. Section~\ref{sec:Architecture} introduces the transformer-based FL framework and an incentive mechanism. Section~\ref{sec:AlgorithmOptimizations} presents the algorithm optimizations applied in the framework and their analysis results. Section~\ref{sec:RelatedWorks} reviews works related to the current study. Finally, the paper concludes in Section~\ref{sec:Conclusions}.

%%%%%%%%%%%%%%%%%%%% Background %%%%%%%%%%%%%%%%%%%%
\section{Background}
\label{sec:Background}

There are various methods for optimizing training and inference in machine learning (ML) models. In this paper, we mainly focus on two areas: model-parallel FL and algorithmic optimizations based on the transformer model architecture. In this section, we present a brief introduction on FL and the transformer model.

\subsection{Federated Learning}
FL \cite{li2020review, zhang2021survey} has emerged as a crucial strategy to optimize training and inference for ML models, especially when dealing with data privacy concerns and distributed computational resources \cite{li2021survey, truex2019hybrid, wang2019adaptive, abdelmoniem2023refl}. FL can be primarily categorized into two types as shown in Figure~\ref{fig:two_type_FL}: data-parallel and model-parallel FL. 

\begin{figure}[ht]
  \centering
    \begin{subfigure}[b]{0.48\textwidth}
        \centering
        \includegraphics[width=\textwidth]{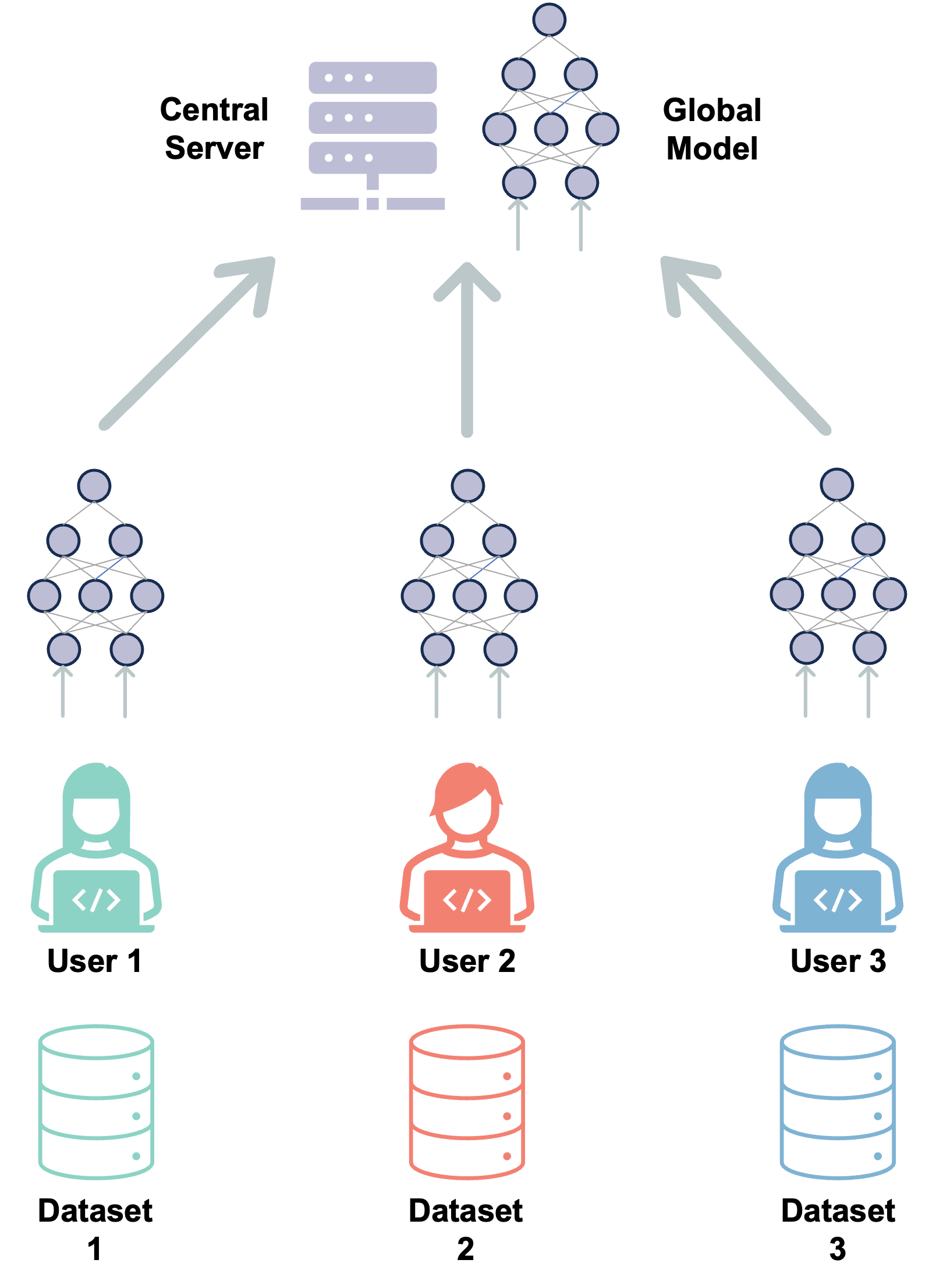}
        \caption{Data-Parallel FL}
    \end{subfigure}
    \hfill
    \begin{subfigure}[b]{0.34\textwidth}
        \centering
        \includegraphics[width=\textwidth]{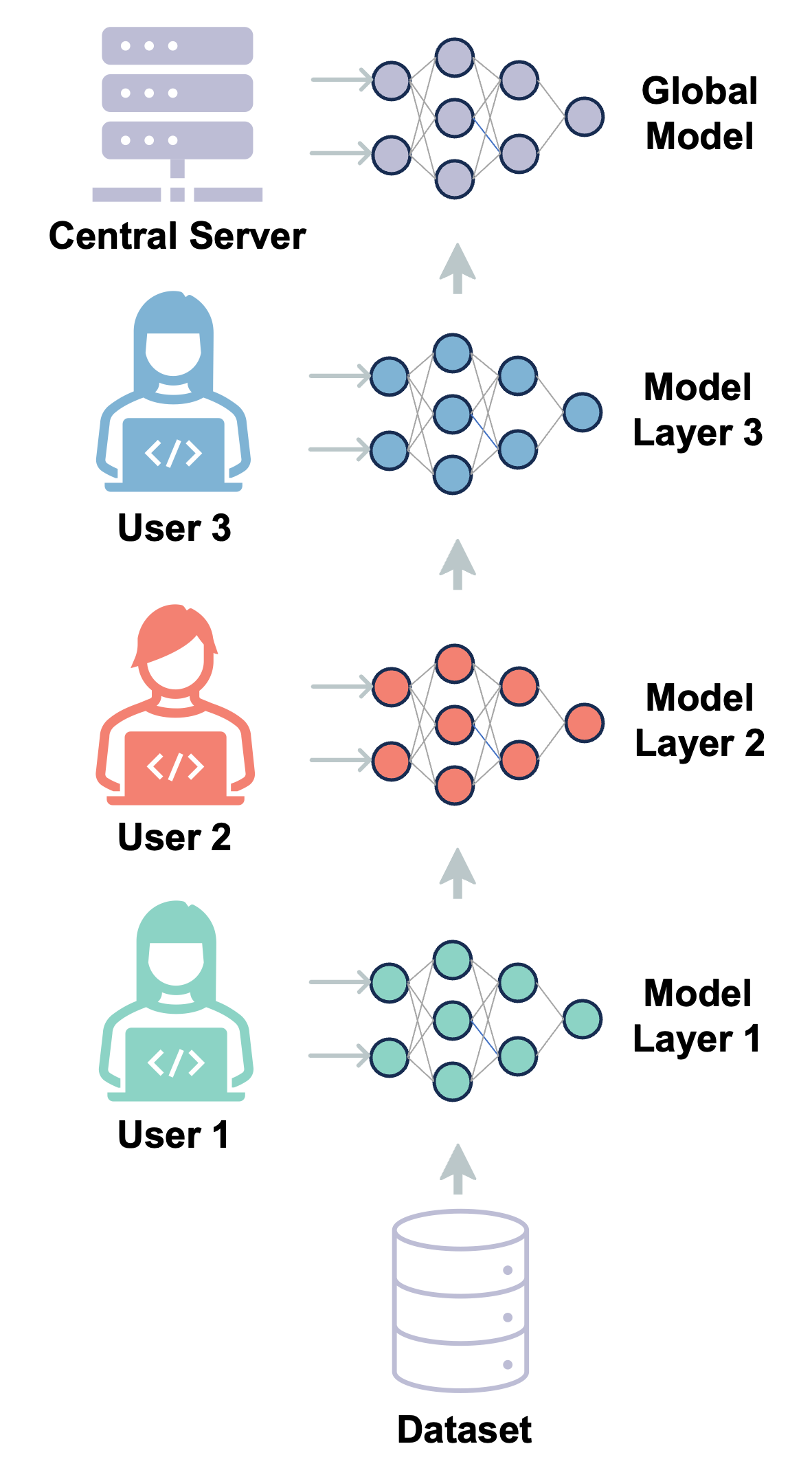}
        \caption{Model-Parallel FL}
    \end{subfigure}  
  \caption{Two Types of Federated Learning Models}
  \label{fig:two_type_FL}
  \Description{}
\end{figure}

Data-parallel FL involves distributing different subsets of training dataset across multiple devices or servers. Each participant trains a complete copy of the model independently on their subset of the data, and the individual updates are then aggregated to update the global model as seen in Figure~\ref{fig:two_type_FL} (a). This approach, while straightforward, places high demands on computational resources, as each device must have the capability to train a complete model. In contrast, model-parallel FL, which is employed in this work, distributes different layers or sections of a model across multiple devices or servers. This approach allows each participant to train only a portion of the model, making it particularly advantageous for large models like LLMs that have extensive parameters. Here, the output of one layer processed by a server becomes the input for the next layer processed by another server in the sequence. The advantages of using these chain-processing FL include enhanced data privacy, as data does not need to leave its original location, and increased utilization of distributed computational resources as presented in Figure~\ref{fig:two_type_FL} (b). 

However, FL also introduces potential security issues. For instance, a participant might upload incorrect or maliciously modified training results, a form of attack known as model poisoning attack \cite{fang2020local, cao2019understanding}. Such an attack can degrade the model's performance or cause it to behave in unintended ways. Furthermore, the security vulnerabilities specific to LLMs, including those trained using FL approaches, are significant and multifaceted. As detailed in \cite{liu2024exploring}, LLMs can be susceptible to various forms of prompt hacking and adversarial attacks, which can subtly manipulate model behavior or inject malicious content. The CFSafety benchmark introduced in \cite{liu2024cfsafety} provides a  framework for assessing the safety and security of LLMs by testing against a wide array of safety scenarios and instruction-based attacks. To address these security concerns, an incentive mechanism is included within the transformer-based FL framework in this paper in  Section~\ref{sec:Architecture}. This mechanism rewards constructive contributions while filtering out malicious activities. 

\subsection{Transformer Model}
Transformers represent a significant architecture in LLMs. Central to the transformer models are two distinct components: the encoder and the decoder, as depicted in Figure~\ref{fig:Transformer}. 

\begin{figure}[ht]
  \centering
  \includegraphics[width=0.75\textwidth]{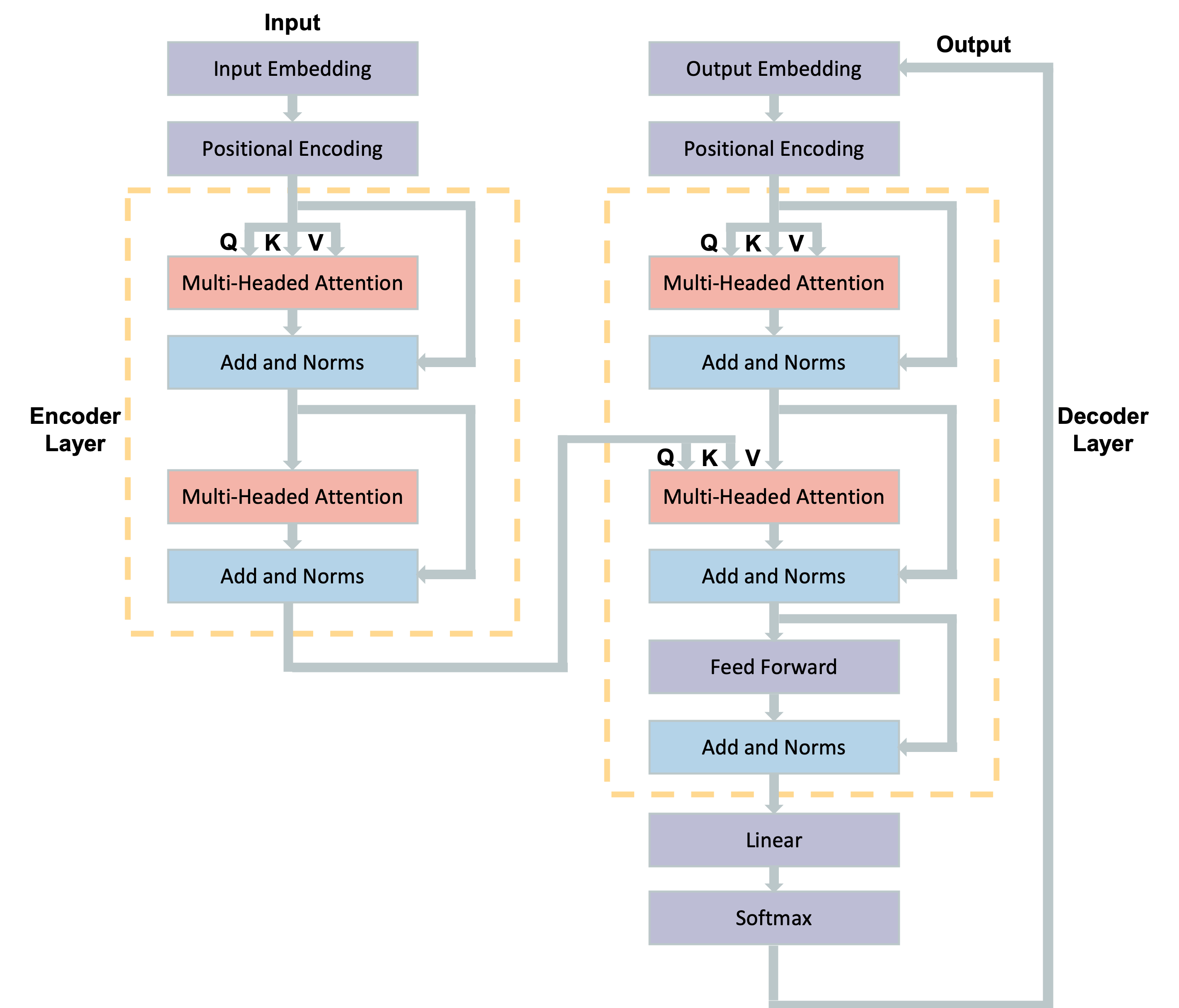}
  \caption{The Transformer Model in LLM}
  \label{fig:Transformer}
  \Description{}
\end{figure}

The encoder processes the input text data by first converting it into embeddings that integrate positional encoding to retain the sequence order. The formula for positional encoding can be expressed as:

\begin{equation}
    PE_{(pos, 2i)} = \sin\left(\frac{pos}{10000^{\frac{2i}{d}}}\right)
\end{equation}

\begin{equation}
    PE_{(pos, 2i+1)} = \cos\left(\frac{pos}{10000^{\frac{2i}{d}}}\right)
\end{equation}

where \(pos\) is the position, \(i\) is the dimension index, and \(d\) is the embedding dimension.

This is followed by a series of operations that enhance the model's ability to focus on different parts of the input data sequence for better context understanding. Multi-Head attention mechanism allows the model to weigh and relate various tokens in the input sequence differently, aggregating information from different representation at different positions. Each head in the multi-head attention mechanism utilizes its own weight matrices for transforming query, key, and value vectors. As the number of heads and the dimension of these vectors increase, the cumulative size of these matrices can become substantial. Optimizing these weight matrices, which is discussed in Section~\ref{sec:AlgorithmOptimizations}, can help reduce the memory footprint and computational overhead. Each attention output is then normalized and added back to the original input through a residual connection, helping to stabilize the learning process. Then, each position is processed independently in the same way through a feed-forward neural network, further transforming the representations. The output from the encoder is a continuous vector representations of the input, enriched with attention-driven contextual information.

The decoder is designed to generate output tokens by transforming the encoded information. Similar to the input embedding, the decoder side processes the input data by output embedding to retain the sequence order. It is also followed by multi-head attention, add \& norm, and a feed-forward network. The final transformation of decoder outputs into a probability distribution over possible output tokens. This also involves a weight matrix, with the $Softmax$ function providing the probabilities necessary for prediction.

%%%%%%%%%%%%%%%%%%%% Architecture %%%%%%%%%%%%%%%%%%%%
\section{Architecture}
\label{sec:Architecture}

The primary applications of pre-trained LLMs are mainly two-fold: inference \cite{schick2021generating} and fine-tuning \cite{hu2021lora, houlsby2019parameter}. Due to its widespread use in real-time applications, the focus here is primarily on the inference aspect of LLMs. Inference involves encoding the input tokens and predicting the next possible tokens. This is explored through an effective transformer-based FL framework, as shown in Figure \ref{fig:FL_framework}, which is designed to optimize the efficiency of inference processes. To solve the aforementioned issues in existing LLM inference, this section will detail the architecture and the operational workflow of each type of users in our framework.

\begin{figure}[b]
  \centering
  \includegraphics[width=0.70\textwidth]{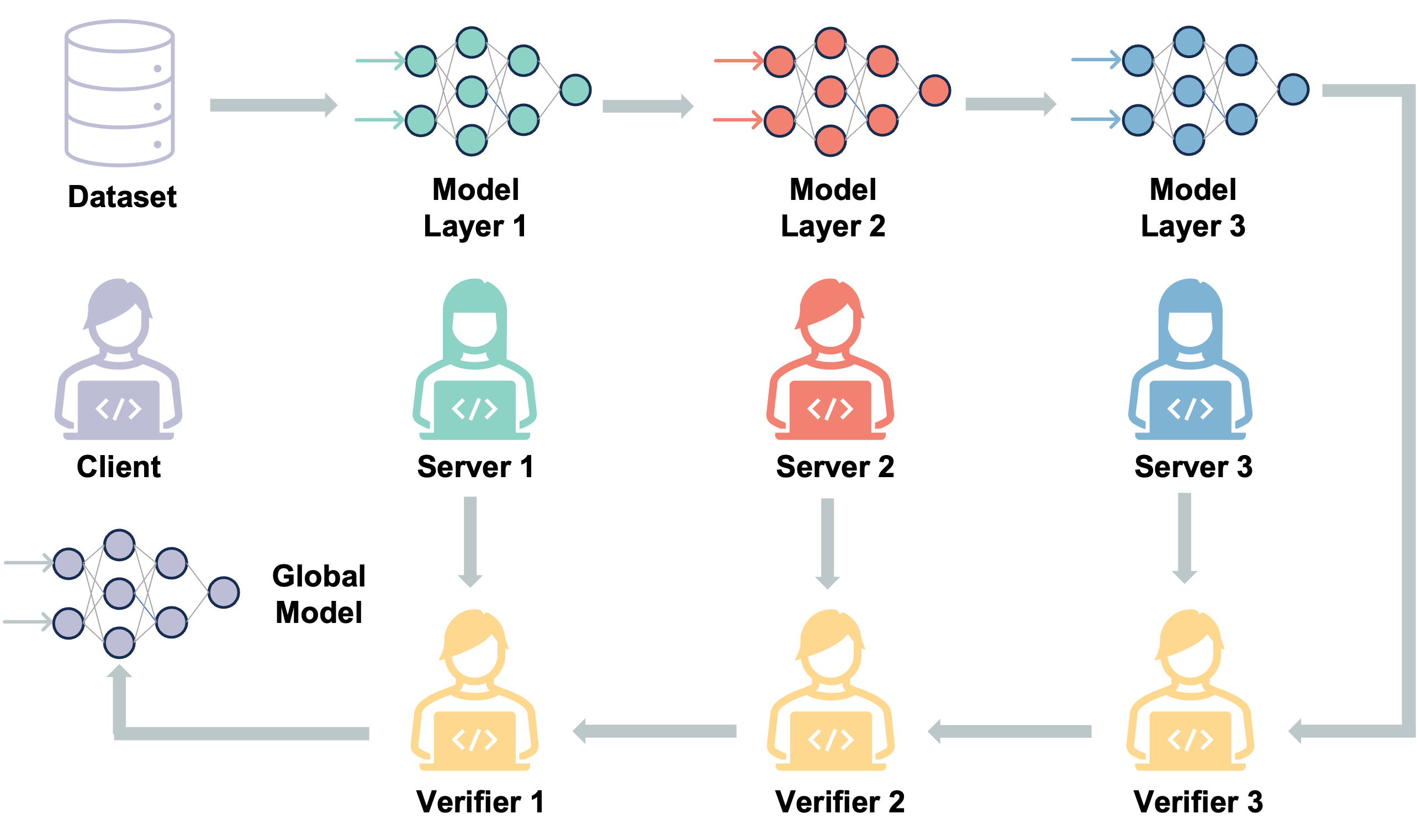}
  \caption{eFedLLM Framework Overview}
  \label{fig:FL_framework}
  \Description{}
\end{figure}

\subsection{Transformer-Based Model-Parallel Federated Learning}

Integral to {\itshape eFedLLM} is the transformer architecture, where each layer consists a complete transformer block, featuring components such as self-attention, a feed-forward network, normalization layers, and residual connection \cite{vaswani2017attention, lyu2019advances}. For an LLM with $L$ layers of transformer block, predicting a sequence of $M$ tokens involves passing each token through all these $L$ layers,  resulting in a time complexity of $O(L \times M)$ during inference \cite{kaplan2020scaling, lyu2019advances}. Furthermore, existing inference techniques often require storing past keys and values in accelerator (global) memory, which is resource-intensive. For instance, maintaining half-precision activation of a 2048-token sequence for LLMs like GPT-3 can consume approximately 9.6 GB of GPU memory per sequence\cite{brown2020language}. To address these challenges, {\itshape eFedLLM} employs transformer-based model parallel FL, aiming to reduce the memory footprint and enhance the feasibility of using advanced LLMs across multiple computing environments.

Building on this, the computational architecture of our framework is designed to distribute the inference workload across an interconnected network of GPUs or servers. This configurations allows different stakeholders in the FL setup—namely $Clients$, $Servers$, and $Verifiers$—to interact efficiently and reliably, ensuring that the computational and memory demands are managed effectively across the network: 

\begin{itemize}
\item {\texttt{$Clients$}}: $Clients$ possess a dataset and a pre-trained LLM but often lack the computational resources necessary for efficient inference. These users initiate the inference tasks within the FL framework by providing the necessary input tokens and their associated embeddings. The initialized parameters of the pre-trained model are also sent to the FL nodes. To handle the computationally intensive aspects of LLM inference, $Clients$ delegate the execution of transformer blocks to $Servers$, who perform the model-parallel computations for processing the input and intermediate data (e.g. model parameters, accumulated gradients) through the layers of the LLMs. After successful training and verification, $Clients$ aggregate the updates from the FL process to continually refine and update the global model. The tasks are repeated until the desired level of convergence is achieved.  
\item {\texttt{$Servers$}}: This group of users are the majority of this system equipped with the requisite GPU resources and are arranged in a sequence to process the computations forwarded by $Clients$. They are active participants in the transformer-based model-parallel FL process, which enhances scalability and resource efficiency for LLM inference. Initially, a specific threshold of hardware resources is established to select participants with adequate computational power. The process begins with the first $Server$, which receives embedding data from a $Client$ and processes the initial layer of the LLM. Subsequent $Server$ sequentially handle the remaining layers of the transformer block. To ensure the reliability and integrity of the FL process, each $Server$'s output is scrutinized by $Verifiers$ in a decentralized manner. By optimizing the use of computational resources across a distributed network, the resource required per user can be reduced and the inference of LLMs can be more economically feasible for a broader range of users. 
\item {\texttt{$Verifiers$}}: It is also crucial to ensure that all the $Servers$ participating the FL have no malicious intent, contributing positively to the federated system. Therefore, to make sure the integrity and reliability of the system, an incentive mechanism is utilized. $Verifiers$ monitor the behavior of each $Server$ by calculating a trust value based on predefined algorithm before the aggregation of updates to the global model. $Server$ that meet or exceed a certain trust threshold are allowed continued access to the global model. Conversely, $Servers$ that fail to meet this threshold are excluded from the system, with their computational tasks reassigned to other trusted $Servers$. 
\end{itemize}

\subsection{Incentive Mechanism}
Incentive mechanisms play a crucial role in ensuring the integrity and reliability of training results within FL environments. To ensure correct training results and reliable intermediate outputs, an incentive mechanism is applied across the server network. Since each layer of the model does not produce a final output that can be directly evaluated, intermediate outputs must be validated. One approach is to establish an expected output for layer $i$ based on historical data, providing a benchmark against which actual outputs can be compared. Alternatively, a set of validation data could be processed through the model layers computed by trusted $Verifiers$. Then the $Verifiers$ can estimate the accuracy $acc_i$ of layer $i$ by comparing the intermediate outputs from layer $i$ against its expected outputs. 

In this incentive mechanism, a $Trust Score$ is calculated for each $Server$. The $Verifiers$ utilize the Trust Score algorithm, denoted as $S_i$, to quantify the reliability of the $i^{th} Server$ within the model-parallel FL framework. It is computed as follows:

\begin{equation}
    \text{Trust Score (S)}_i = \frac{\text{acc}_i \times l_i}{\text{max(l)}} \times \text{w}_i
\end{equation}

where $acc_i$ represents the accuracy achieved by the $i^{th} Server$ on its assigned tasks, indicating the $Server$'s proficiency in processing its portion of the model. $l_i$ denotes the number of model layers that the $i^{th} Server$ is responsible for. $Server$ with more substantial computational resources may handle more layers, contributing to a higher computational load. $max(l)$ is the maximum number of layers any $Server$ within the network is handling, which serves to normalize the effect of layer distribution across all $Servers$, ensuring that their contributions are evaluated on a comparable scale. $w_i$ is a weighting factor applied to adjust and ensure that the $Trust Score S_i$ remains bounded between 0 and 1.   

Additionally, a threshold $\theta$ is implemented to manage the participation of $Servers$ in the FL network. This threshold is crucial for filtering out potentially malicious or under-performing $Servers$ while incentivizing those that meet or exceed the expected performance standards. Specifically, for each $Server$, the $Trust Score$ is compared to the threshold $\theta$ as shown in the following formula. If $\text{Trust Score (S)}_i$ is equal to or greater than $\theta$, the $Server$ remains active and may receive incentives for its contributions. Conversely, if $\text{Trust Score (S)}_i$ is less than $\theta$, the $Server$ is deactivated to prevent any negative impact on the model's training and inference processes. The tasks of the deactivated $Server$ will be reassigned to other qualified $Servers$. This dynamic verification is critical for maintaining the overall reliability and efficiency of the distributed learning environment. 

\begin{equation}
    \text{Server Status (i)} = 
    \begin{cases}
        \text{Trust Score (S)}_i \geq \theta & \text{activate} \\
        \text{Trust Score (S)}_i < \theta & \text{deactivate}
    \end{cases}
\end{equation}

%%%%%%%%%%%%%%%%%%%% Algorithm Optimizations %%%%%%%%%%%%%%%%%%%%
\section{Algorithm Optimizations}
\label{sec:AlgorithmOptimizations}

This section discusses the advanced optimizations implemented within the transformer algorithms of the {\itshape eFedLLM} framework, targeting enhanced operational efficiency. The focus is primarily on three critical areas: optimizing matrix multiplication using memory hierarchy strategies, improving matrix transfer efficiency through Singular Value Decomposition (SVD), and enhancing verification processes. By refining these processes, we aim to streamline communications, reduce latency and boost the overall efficiency in this system. 

Before delving into the specific optimizations, it is essential to establish a common understanding of the notations used throughout this section. Table~\ref{notation} summarizing the key symbols and their meanings. With these notations defined, the subsequent discussion explores the implementation of each optimization strategy.

\begin{table}[]
\caption{Notation Table}
\label{notation}
\begin{tabular}{ll}
\toprule
Notation & Meaning \\
\midrule                           
\( R_t \) & Reduction in memory read time when using Federated Learning over a centralized model \\
\( T_c \) & Memory read time in a centralized model \\
\( T_f \) & Memory read time in a Federated Learning model \\
\( W \) & Weight matrix in Federated Learning \\
\( W_k \) & Low-rank approximation matrix of the weight matrix \\
\( U_k \) & Matrix of left singular vectors after retaining \( k \) singular values \\
\( \Sigma_k \) & Diagonal matrix after retaining \( k \) singular values \\
\( V_k^\top \) & Matrix of right singular vectors after retaining \( k \) singular values (transposed) \\
\( P \) & Cumulative energy ratio, accuracy of low-rank approximation \\
\( \sigma_i \) & The \( i^{th} \) singular value of the matrix \\
\( e \) & Desired accuracy of low-rank approximation matrix to be retained \\
\( L \) & Total number of layers \\
\( H \) & Total number of attention heads \\
\( P_{total} \) & Overall final accuracy \\
\( e_{ij} \) & Expected accuracy for each attention head \( j \) at layer \( i \) \\
\( \hat{W} \) & Weight matrix after truncated \\
\( \hat{U} \) & Matrix of left singular vectors after truncated \\
\( \hat{V}^T \) & Matrix of right singular vectors after truncated \\
\( \hat{\sigma_i} \) & Diagonal matrix after truncated \\
\( Z \) & Represents \( QK^T \), the dot product of \( Q \) and \( K \) \\
\( Z' \) & The result after subtracting a constant \( \hat{z}_v \) from \( Z \) \\
\( \hat{z}_v \) & A constant used for the shift-invariance property of the softmax function \\
\( Y \) & The result of applying negative K-digit base-b number transformation to \( Z' \) \\
\( b \) & The base used for K-digit number transformation \\
\bottomrule
\end{tabular}
\end{table}

%%%%%%%%%%%%%%%%%%%% Matrix Multiplication Optimization %%%%%%%%%%%%%%%%%%%%
\subsection{Matrix Multiplication Optimization}
Matrix operation and data transfer are two computationally intensive tasks central to the performance of LLM operations, particularly in the context of transformer architectures. The conventional centralized learning approach often suffers from inefficiencies due to frequent and extensive data reads from global memory during matrix operations. Each element of the resultant matrix in a matrix multiplication involves multiple reads from the memory for each element calculation, significantly increasing the time and computational overhead.  

FL, contrasted with centralized approaches, can optimize this process by leveraging distributed computational resources. By employing a hierarchy structure depicted in Figure \ref{fig:Hierarchy}, each handling a segment of the computation, federated learning reduces the frequency of memory reads. The matrices are read once globally, and intermediate results are temporarily stored in faster, locally accessible block memory. The final results are then compiled from these intermediate states and written back to the global memory.

Figure \ref{fig:Hierarchy} represents the hierarchical memory structure utilized in a FL environment where GPUs are grouped into blocks. Each block contains multiple GPUs, and they collaboratively work on specific layers of a Transformer model. The data flow starts from the global memory, which is slower compared to block memory. During the computation phase, each block reads data from the global memory. The GPUs within a block process this data by performing computations such as forward and back propagation, and gradient calculations, all of which are part of training a ML model like a Transformer. This computation occurs locally within each block to minimize the latency involved in accessing the slower global memory. After processing, intermediate results are stored temporarily within the block memory. Once all blocks have processed their respective data, the aggregated results are written back to the global memory. This setup significantly enhances efficiency by reducing the frequent read and write operations to the slower global memory, aligning with the principles of federated learning to distribute computational tasks efficiently across multiple devices.

\begin{figure}[ht]
  \centering
  \includegraphics[width=0.70\textwidth]{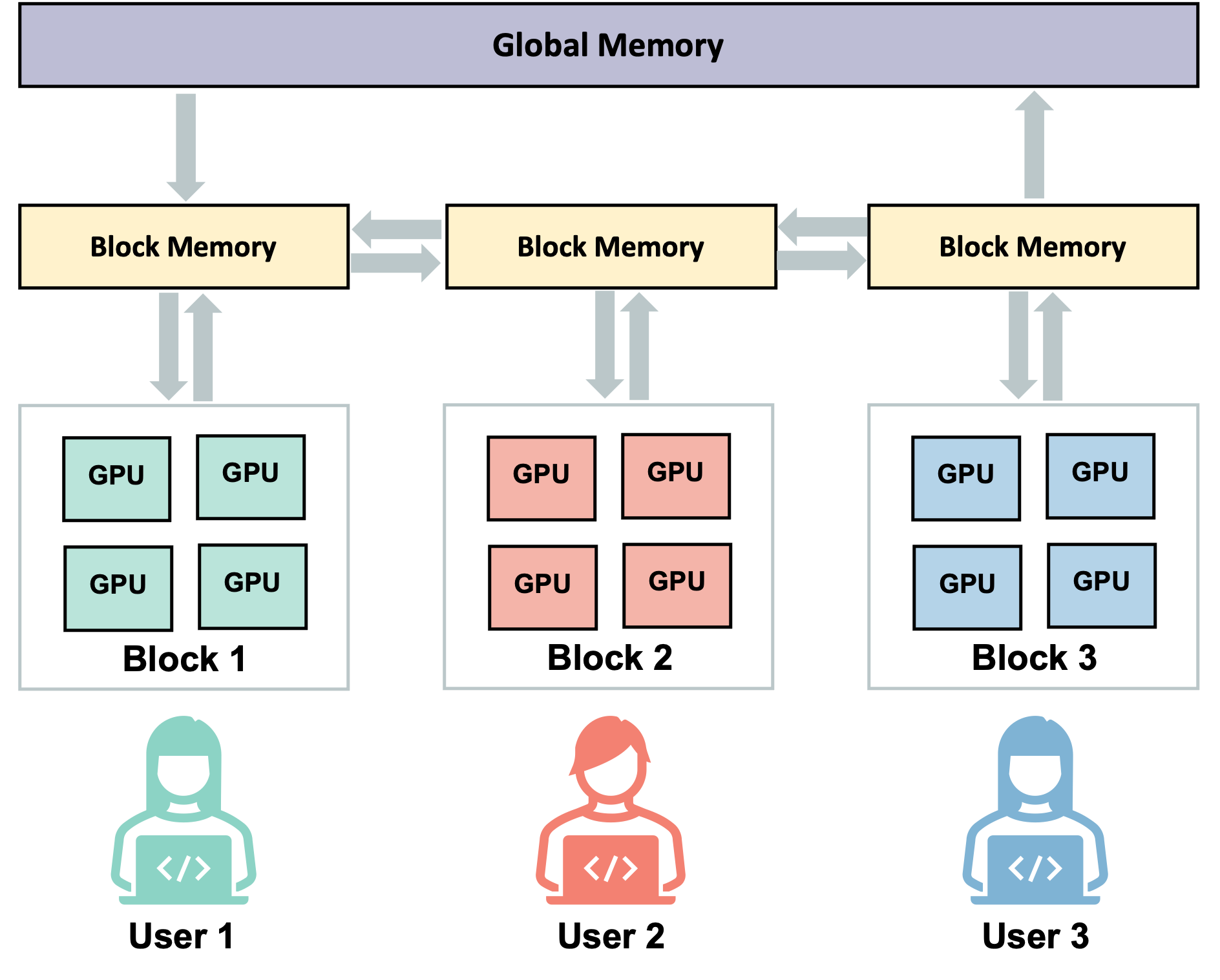}
  \caption{Memory Hierarchy}
  \label{fig:Hierarchy}
  \Description{}
\end{figure}

For example, consider the product matrix $C = A \times B$ in equatio~\ref{example1}, where $c_{ij} = a_{i1}b_{1j} + a_{i2}b_{2j} + a_{i3}b_{3j}$. In centralized learning environments, the computation of each element of $C$ necessitates repeated reads from the global memory. For $c_{11}$, it requires reading three elements from row one of matrix $A$ and three corresponding elements from column one of matrix $B$, totaling six memory reads. Given that matrix $C$ is a $3 \times 3$ matrix, computing all its elements will involve $6 \times 9 = 54$ memory reads from global memory. 

% Efficiency - GEMM
\begin{equation}
    C = A \times B = 
    \begin{pmatrix} 
    a_{11} & a_{12} & a_{13} \\ 
    a_{21} & a_{22} & a_{23} \\ 
    a_{31} & a_{32} & a_{33} 
    \end{pmatrix}
    \times
    \begin{pmatrix} 
    b_{11} & b_{12} & b_{13} \\ 
    b_{21} & b_{22} & b_{23} \\ 
    b_{31} & b_{32} & b_{33} 
    \end{pmatrix}
\label{example1}
\end{equation}

In contrast, FL significantly reduces memory read operations by utilizing distributed computation. In this approach, entire matrices $A$ and $B$ are read from the global memory just once each, leading to a total of $9 + 9 = 18$ reads. The intermediate results—specific to each part of the computation—are stored in faster, local block memory and are later aggregated. 

Table~\ref{time_compare_Two_Matrices} exemplifies the drastic reduction in memory read times, highlighting the scalability and efficiency gains in federated settings. As matrix dimensions increase, the benefits of federated learning become even more pronounced, showcasing significant reductions in memory read operations compared to centralized learning models. This approach not only speeds up computations but also minimizes energy consumption and potential bottlenecks associated with large-scale data processing in LLMs.

\begin{table}[]
\caption{Memory Read Times Comparison - Two Matrices Multiplication}
\label{time_compare_Two_Matrices}
\begin{tabular}{llll}
\toprule
Matrix Dimension & Centralized Learning & Federated Learning & Reduced Time \\
\midrule                           
5      & 250        & 50        & 80.00\%    \\
10     & 2,000      & 200       & 90.00\%    \\
100    & \(2 \times 10^{6}\)  & 20,000 & 99.00\%    \\
10,000 & \(2 \times 10^{12}\) & \(2 \times 10^{8}\) & 99.99\%    \\
\bottomrule
\end{tabular}
\end{table}

More generally, we have this theorem:
\begin{theorem}
    Let \(A_{m \times n}\) and \(B_{n \times k}\) be matrices, and let \(C = A \times B\) represent the matrix multiplication. Then, when deploying a Federated Learning model over a centralized model, the reduction in memory read time, \(R_t\), can be expressed as: 
    \begin{equation}
    R_t = 1 - \frac{1}{2k} - \frac{1}{2m}
    \end{equation}
\end{theorem}  
\begin{proof}
    Let the centralized read time \( T_c \) be given by 
    \[
    T_c = (n + n) m k = 2nmk
    \]  
    which accounts for reading each matrix element separately for each element of the result matrix.

    In Federated Learning approach, the read time \( T_f \) can be optimized to 
    \[
    T_f = m n + n k
    \]  
    since each matrix is read only once per batch and the intermediate results are stored and reused.

    Thus, the reduction in read time \( R_t \) can be calculated as:
    \[
    R_t = \frac{T_c - T_f}{T_c} = \frac{2nmk - (mn + nk)}{2nmk} = \frac{2mk - m - k}{2mk} = 1 - \frac{1}{2k} - \frac{1}{2m}.
    \]
\end{proof}

In the Transformer model, the attention mechanism's $Softmax$ function and the feed-forward networks represent the most computationally intensive operations, both involving matrix multiplication. In the attention mechanism, matrix multiplication appears in computing $QK^T$, which is essential for determining the attention scores. In the feed-forward networks the operation often involves multiplying the input matrix $X$ sequentially with two weight matrices. Our hierarchical computational structure becomes particularly advantageous when dealing with the multiplication of three or more matrices.

%%%%%%%%%%%%%%%%%%%% Other Operations and Algorithms %%%%%%%%%%%%%%%%%%%%
The exploration of the Transformer model also extends to evaluating the impact of addition, subtraction, and division operations on matrices. However, these operations have not been the focus of our optimization efforts within the framework. This decision stems from the observation that in addition and subtraction operations, each element of the resulting matrix necessitates accessing the corresponding elements in the original matrices. The number of memory accesses required remains consistent regardless of whether the model is operating under centralized or federated learning paradigms with a hierarchical memory structure. Consequently, the potential for performance improvement through optimization in addition and subtraction is limited. Additionally, matrix division is not directly included in transformer models. The only division-like operation in transformer models occurs within the $Softmax$ function of the attention mechanism. This operation involves exponentiation of each element of the input matrix and then normalizing these values to ensure that the output sums to one, which is a form of scalar division across elements rather than matrix division. 

In addition, we do not include optimizations for layer normalization and residual connections because their computational demands are relatively minimal compared to the more intensive operations of the attention mechanism and feed-forward networks. Furthermore, both layer normalization and residual connections primarily involve element-wise operations—specifically addition and multiplication—rather than matrix operations. These operations are designed to stabilize the learning process and facilitate the flow of gradients through the network during training, but they do not require the complex and resource-intensive matrix calculations found in other parts of the transformer architecture.

%%%%%%%%%%%%%%%%%%%% Data Transfer Optimization %%%%%%%%%%%%%%%%%%%%
\subsection{Matrix Transfer Optimization}

Federated Learning (FL) requires frequent exchange of large model parameters or gradients between distributed computing devices, making network transmission one of the major limitations of FL. The high communication overhead, bandwidth limitations, latency, and network instability can hinder the effective implementation of FL.To address these challenges, a method is proposed that combines Singular Value Decomposition (SVD) to reduce the amount of transmitted data \cite{10.1145/3534678.3539402, doi:10.1073/pnas.2024789118}, thereby reducing communication latency and improving overall FL efficiency.

Given a weight matrix $W \in \mathbb{R}^{m \times n}$, which represents the parameters that need to be transmitted in Federated Learning, Singular Value Decomposition (SVD) is first applied to the matrix $W$:

\begin{equation}
    W = U \Sigma V^\top
\end{equation}

where $U \in \mathbb{R}^{m \times m}$ and $V \in \mathbb{R}^{n \times n}$ are orthogonal matrices, and $\Sigma \in \mathbb{R}^{m \times n}$ is a diagonal matrix containing the singular values of the matrix $W$. To reduce the dimensionality, only the largest $k$ singular values and their corresponding singular vectors are retained, resulting in a low-rank matrix.:

\begin{equation}
    W_k = U_k \Sigma_k V_k^\top
\end{equation}

where $U_k \in \mathbb{R}^{m \times k}$, $\Sigma_k \in \mathbb{R}^{k \times k}$, and $V_k \in \mathbb{R}^{n \times k}$.

The low-rank matrices $U_k$, $\Sigma_k$, and $V_k^\top$ are transmitted to other devices. Upon receiving these matrices, the devices use them to reconstruct an approximation of the original matrix $W_k$. The low-rank approximation $W_k$ is the matrix that retains the first $k$ singular values, and the retained energy is $\sum_{i=1}^{k} \sigma_i^2$, where $\sigma_i$ is the singular value of matrix, so the accuracy retained by the low-rank approximation \( W_k \) can be estimated using the cumulative energy ratio. The retained accuracy is defined as:

\begin{equation}
P = \frac{\sum_{i=1}^{k} \sigma_i^2}{\sum_{i=1}^{r} \sigma_i^2}
\end{equation}

where \( r = \min(m, n) \).

The reconstruction accuracy depends on the number of singular values retained during the SVD process. Since the singular values tend to decay rapidly. In practice, retaining the top 40\%-50\% of the singular values typically preserves 90\%-99\% of the matrix's energy. 

To understand the efficiency gains of using SVD in FL, an analysis of the complexity of SVD and the communication savings is required.

Performing SVD on a matrix \( W \in \mathbb{R}^{mn} \) has a computational complexity of \( O(\min(m^2n, mn^2)) \). To expedite the low-rank process, we have decided to employ the fast SVD algorithm developed by Chen Xu, Weiwei Xu \cite{10.1093/nsr/nwad083}, which achieves a computational complexity of 
\( O(mnr) \). 

Despite the high computational cost associated with SVD for large matrices, it is executed only once per communication round. The total number of elements to be transmitted is \( m \times n \). After performing SVD low-rank approximation on an \( m \times n \) matrix, the transmitted matrices are \( U_k \in \mathbb{R}^{m \times k} \), \( \Sigma_k \in \mathbb{R}^{k \times k} \), and \( V_k^\top \in \mathbb{R}^{k \times n} \). Therefore, the total size of the transmitted data is \( mk + k^2 + kn \). When \( k \ll m \) and \( k \ll n \), the reduction in transmission bandwidth can be approximated as \( (m-k)(n-k) \). Similarly, retaining the top 40\% to 50\% of the singular values can save 40\% to 60\% of the bandwidth while maintaining a high level of accuracy.

In this case, even though SVD itself has a certain computational complexity, it can still bring significant efficiency improvements to federated learning in multiple rounds of communication.

\begin{figure}[t]
  \centering
  \includegraphics[scale=0.07]{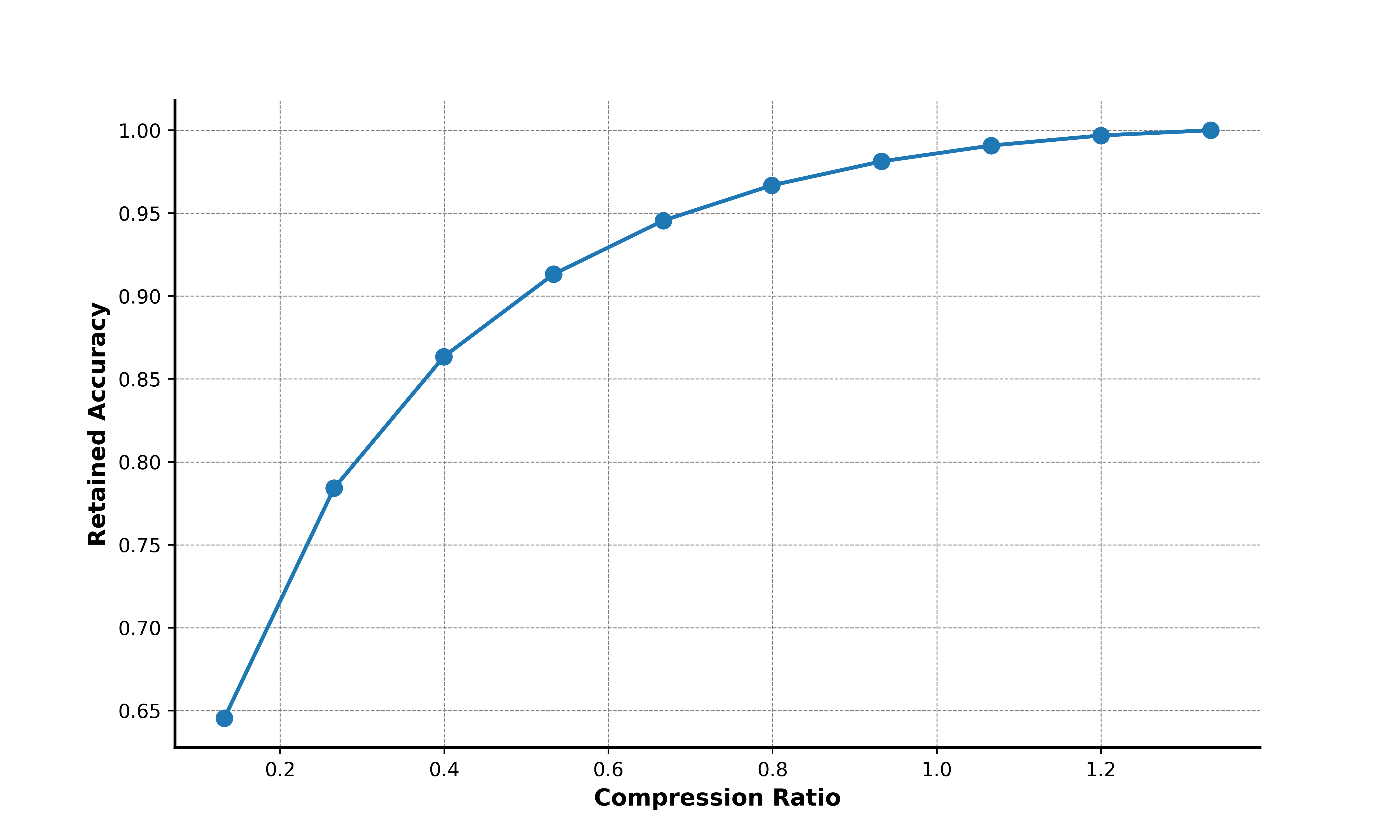}
  \caption{The relationship between the compression ratio and the retained accuracy when preserving different numbers of singular values.}
  \label{fig:example}
  \Description{The relationship between the compression ratio and the retained accuracy when preserving different numbers of singular values.}
\end{figure}

The compression ratio can be calculated as: 
\begin{equation}
    Compression Ratio = \frac{Compressed Matrix Size}{Original Matrix Size} = \frac{(m + n + 1) \times k}{m \times n}
\end{equation}

To verify the effectiveness of the proposed method, Singular Value Decomposition (SVD) was applied to the \texttt{h.1.attn.c\_attn.weight} layer of the GPT-2 model, which originally had a size of \texttt{(768, 2304)}. After retaining the top 40\% of the singular values, the compression ratio reached 53.32\%, while the reconstructed matrix retained 91.32\% of the original matrix's energy. Furthermore, the relationship between the compression ratio and the retained accuracy was tested by preserving different numbers of singular values, and the results are shown in Figure~\ref{fig:example}. This indicates that the SVD method can significantly reduce the amount of transmitted data while maintaining a high level of model accuracy.

These results further demonstrate the potential advantages of SVD in federated learning. Even when retaining only a portion of the singular values, it effectively reduces communication overhead while preserving high model accuracy. This approach can notably improve the overall efficiency of federated learning, especially in scenarios involving multiple communication rounds.

Moreover, this proposition~\ref{proposition} outlines that applying SVD to weight matrix $W$ can approximate it to lower rank while maintaining a considerable amount of the matrix's original information:
\begin{proposition}
    Let weight matrix $W \in \mathbb{R}^{m \times n}$. If $W = U \Sigma V^T$ is the SVD form of $W$, for any $\epsilon > 0$, matrix $W$ can be approximated by a rank $k$ matrix $W_k = U_k \Sigma_k V_k^T$ to a relative tolerance $\epsilon > 0$ in the sense that: there exists rank-k matrix $W_k$, with $k \ge C|log\epsilon|$ for some constant $C$, such that:
    \begin{equation}
        \|W - W_k\| \leq \epsilon
    \end{equation}
\label{proposition}
\end{proposition} 

where $k$ is chosen such that the sum of the discarded singular values squared is less than or equal to $\epsilon^2$. This approach is valuable in contexts such as LLM inference, where efficient data handling is necessary without compromising the accuracy of the data. 

In a complete Transformer model, federated learning typically involves dynamically splitting the model architecture and distributing it across different devices, depending on their computational power and network bandwidth. For the purpose of theoretical analysis, a standard partitioning approach is assumed in this paper. Specifically, one encoder layer is taken as an example, where the embedding layer and position embedding are grouped together; each attention head in the multi-head attention mechanism is handled separately; and finally, the feedforward network is treated as a distinct module applied across the entire Transformer model.

It is assumed that the SVD dimensionality reduction method is applied during the input and output processes of each attention head. The following formula is used to constrain the accuracy and determine the number of singular values retained during each reduction.

\begin{equation}
    \sum_{i=1}^{\text{Rank}(W_k)} \sigma_i^2 \geq e \sum_{i=1}^{\text{Rank}(W)} \sigma_i^2
\end{equation}

Where $\sigma_i$ represents the singular values of weight matrix $W$, and $e$ denotes the desired accuracy to be retained.

After processing through multiple encoder and decoder layers, the overall model accuracy can be approximately represented as the cumulative effect of this method. 

\begin{equation}
    P_{\text{total}} = \prod_{i=1}^{L} \left(\frac{1}{H} \sum_{j=1}^{H} e_{ij} \right)
\end{equation}

Where $P_{\text{total}}$ represents the overall final accuracy, $L$ denotes the total number of layers, $H$ is the number of attention heads, and $e_{ij}$ represents the expected accuracy for each attention head j at layer i.

However, it is important to note that as the number of layers increases, the accuracy retention tends to decrease layer by layer. Therefore, when designing the partitioning strategy, it is crucial to carefully select the dimensionality reduction accuracy e for each layer to balance between bandwidth and accuracy.

%%%%%%%%%%%%%%%%%%%% Combination of Memory Hierarchy and SVD %%%%%%%%%%%%%%%%%%%%
\subsection{Combination of Memory Hierarchy and SVD}
In this subsection, we combine the above methods together to analyze the overall optimization. 

After rank-k compression through SVD, we truncate the weight matrix $W$ according to the compression ratio. Then the truncated weight matrix $\hat{W}$ becomes:
\begin{equation}
    \hat{W} = \hat{U} \times \hat{\Sigma} \times \hat{V}^T
\end{equation}
where $\hat{U} \in \mathbb{R}^{m \times \hat{k}}$, $\hat{\Sigma} \in \mathbb{R}^{\hat{k} \times \hat{k}}$, and $\hat{V} \in \mathbb{R}^{n \times \hat{k}}$. The number of rank after truncated is:
\begin{equation}
    \hat{k} = \frac{m \times n \times Compression Ratio}{m + n + 1}
\end{equation}

For comparing $WX$ and $\hat{W}X$, where $W, \hat{W} \in \mathbb{R}^{m \times n}$ are the weight matrices and $X \in \mathbb{R}^{n \times t}$ is the input, we have the comparison in Table~\ref{combination}.

\begin{table}[b]
\caption{Comparison of Original and Compressed Matrices With or Without Memory Hierarchy}
\label{time_compare}
\begin{tabular}{lll}
\toprule
Metric & Original \( WX \) & Compressed \( \hat{W}X \) \\
\midrule                           
Weight Matrix Storage Requirement & \( mn \) & \( (m+n+1)\hat{k} \) \\
Memory Read of Multiplication without Hierarchy & \( 2mnt \) & \( 2(m+n)\hat{k}t \) \\
Memory Read of Multiplication with Hierarchy & \( mn + nt \) & \( m\hat{k} + \hat{k} + n\hat{k} + nt \) \\
\bottomrule
\end{tabular}
\label{combination}
\end{table}

Consider the first linear layer in a BERT model, where the weight matrix \( W \in \mathbb{R}^{3072 \times 768} \), reflecting a common architecture choice where the model expands the embedding dimension 768 to a higher internal dimension 3072 as part of its feed-forward network. \( X \in \mathbb{R}^{768 \times 30} \) is the input matrix with sequence length equals 30, assuming that each of the 30 tokens is represented by a 768-dimensional vector as output from the previous layer. The multiplication $WX$ represents a core computation in this layer. Applying SVD and memory hierarchy optimization techniques, $WX$  is compressed to $\hat{W}X$. The results of comparing the memory read for $WX$ and $\hat{W}X$ with and without memory hierarchy are shown in Figure \ref{fig:compare}. The compression ratios is varied from 0.2 to 0.8. These results illustrate the impact of memory hierarchy and SVD method on data communication at various compression ratios ranging. In Figure \ref{fig:compare} (a), it is clear that with only the SVD method is applied, the more the matrix size is compressed, the less memory read is needed for $\hat{W}X$. In Figure \ref{fig:compare} (b), it shows a similar trend in reducing memory read but a reduced scale (from 1e8 to 1e6). This demonstrates the effectiveness of memory hierarchy in optimizing memory usage. This becomes particularly advantageous in environments with limited memory resources or where memory bandwidth is a bottleneck.

\begin{figure}[h]
  \centering
    \begin{subfigure}[b]{0.47\textwidth}
        \centering
        \includegraphics[width=\textwidth]{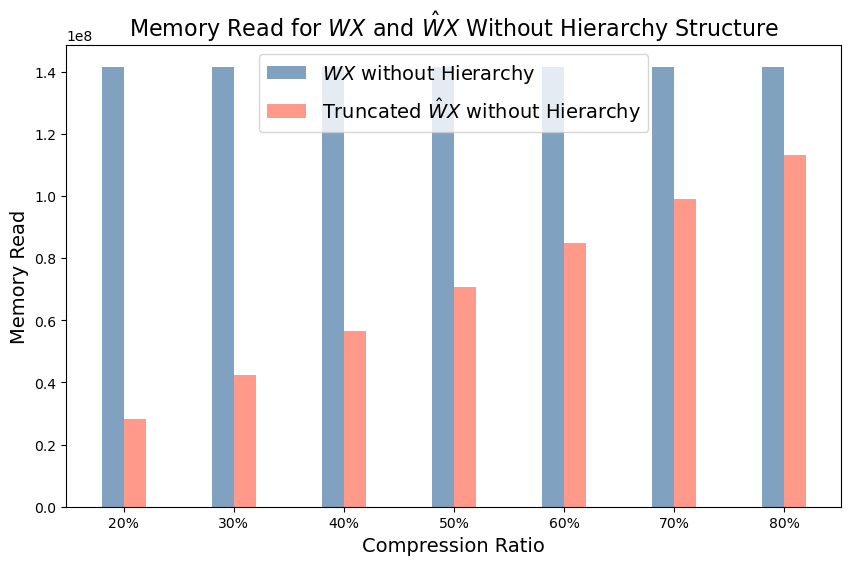}
        \caption{Memory Read Without Hierarchy}
        \label{fig:memory_read_without_hierarchy}
    \end{subfigure}
    \hfill
    \begin{subfigure}[b]{0.47\textwidth}
        \centering
        \includegraphics[width=\textwidth]{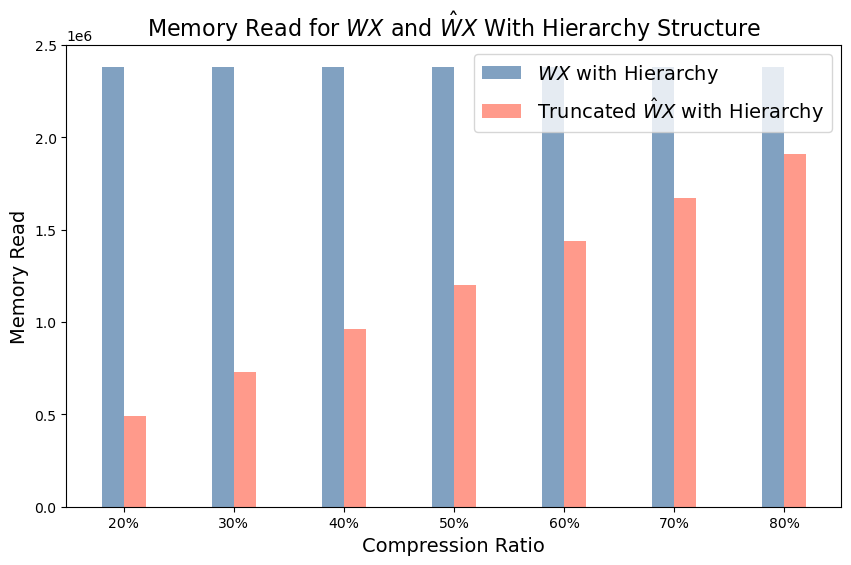}
        \caption{Memory Read With Hierarchy}
        \label{fig:memory_read_with_hierarchy}
    \end{subfigure}  
  \caption{Memory Read Comparative Analysis in BERT Model}
  \label{fig:compare}
  \Description{Comparative Analysis of Memory Read Time with and without Memory Hierarchy Strategies and Singular Value Decomposition (SVD)}
\end{figure}

To analyze the impact of applying both SVD and memory hierarchy techniques on reducing bandwidth usage, bandwidth reduce rate is utilized: 
\begin{equation}
    Bandwidth Reduce Rate = 1 - \frac{Optimized Total Memory Access}{Original Total Memory Access}
\end{equation}
where total memory access is calculated as:
\begin{equation}
    Total Memory Access = Weight Matrix Read + Input Matrix Read + Output Matrix Write
\end{equation}

Again consider the first linear layer of a BERT model. Let batch size is 10 and matrices stored as 32-bit floating point numbers (4 bytes each). The relationship between the compression ratio and the bandwidth reduce rate is illustrated in Figure \ref{fig:Bandwidth Reduce Rate}. By implementing the algorithmic optimizations, it becomes evident that the bandwidth reduce rate decreases as the compression ratio increases. This trend illustrates that greater compression of the matrices leads to more significant reductions in bandwidth usage. Specifically, setting the compression ratio to 0.7 can reduce the bandwidth by 60\% of what it was originally. This outcome confirms the effectiveness of our methods in lowering memory bandwidth demands, consequently improving the computational efficiency of deploying large models such as BERT.

 \begin{figure}[t]
  \centering
  \includegraphics[scale=0.47]{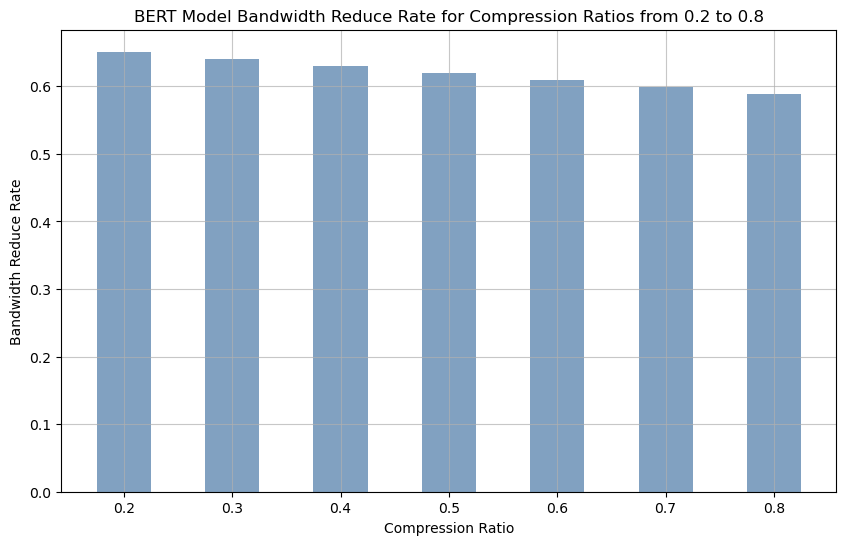}
  \caption{Bandwidth Reduce Rate Analysis in BERT Model}
  \label{fig:Bandwidth Reduce Rate}
  \Description{The relationship between the compression ratio and the bandwidth reduce rate is shown in this figure.}
\end{figure}

%%%%%%%%%%%%%%%%%%%% Verification Optimization %%%%%%%%%%%%%%%%%%%%
\subsection{Verification Optimization}
The $Verifiers$' verification process is also optimized. The optimization is again targeted at where the computation is most intense -- the $Softmax$ function. The $Softmax$ function transforms the score $z_v$ (derived from $QK^T$) into a probability distribution by exponentiating each score and normalizing these values by the sum of all score exponentials. This normalization ensures that the output of the $Softmax$ function sums to one and forms a valid probability distribution. This allows the model to softly select the amount of attention to place on each value in the sequence. Building on this fundamental operation, this subsection details specialized algorithmic enhancements aimed at optimizing the verification process, further refining the efficiency and accuracy of the model’s evaluative capabilities.

By splitting the calculation of $\exp(z_v)$ in equation~\ref{softmax} and the summation across multiple $Verifier$ nodes, the computational burden on any single $Verifier$ can be significantly reduced in a distributed manner. 

% Memory - Attention
\begin{equation}
    \text{Softmax}(z)_v = \frac{\exp(z_v)}{\sum_{v'} \exp(z_{v'})}
\label{softmax}
\end{equation}

% using tlookup in attention
The Attention Mechanism Equation utilized in the Transformer models poses a unique challenge for verification due to its reliance on non-arithmetic operations, which do not compute numerical values directly: 
\begin{equation}
    \text{Attention}(Q, K, V) := \text{Softmax}\left(\frac{QK^T}{\sqrt{d}}\right) V
\end{equation}
Here, $Q$, $K$, and $V$ represent query, key, and value components essential to the model's ability to weigh and prioritize different parts of the input data based on relevance. The operation $QK^T$ computes the dot product, representing how queries are mapped to keys. It is an arithmetic operation which can be verified by existing methods \cite{chiesa2017zero}. Then $QK^T$ is scaled by $\frac{1}{\sqrt{d}}$ (where $d$ is the dimension of queries and keys), and normalized by the $Softmax$ function to generate probability between 0 and 1 to normalize the scores. This sophisticated equation highlights the non-linear and complex interdependencies modeled by the transformer layers.  

To address the efficient verification of such non-arithmetic elements in a high level of parallelization, particularly in the Attention mechanism, a specialized transformation is adopted from \cite{sun2024zkLLM}. 

Firstly, to simplify the expression, let:
\begin{equation}
    Z \leftarrow QK^T
\end{equation}

Next, the shift-invariance property of $Softmax$ is exploited by adding a constant $\hat{z_v}$ to $Z$:
\begin{equation}
    Z' := Z - \hat{z_v}
\end{equation}
It can be simply proved that $Z$ and $Z'$ lead to same result through $Softmax$ function:

\begin{proof}
    \[
    \text{Softmax}(z_v - \hat{z_v}) \\
    = \frac{\exp(z_v - \hat{z_v})}{\sum_{v'} \exp(z_{v'} - \hat{z_v'})} \\
    = \frac{\exp(z_v)\exp( - \hat{z_v})}{\sum_{v'} \exp(z_{v'})\exp( - \hat{z_v'})} \\
    = \frac{\exp(z_v)}{\sum_{v'} \exp(z_{v'})} \\
    = \text{Softmax}(z_v)
    \]
\end{proof}

Then negative K-digit based-b numbers transformation \cite{sun2024zkLLM} is applied: 
\begin{equation}
    Y = \exp(Z') \\
    = \exp\left( -\sum_{k=0}^{K-1} b^k Z(k) \right) \\
    = \prod_{k=0}^{K-1} \exp\left( -b^k Z(k) \right)  
\end{equation}

This transformation allows a complex exponential operation with large exponents to be turned into a cumulative multiplication calculation so that it can continue to be optimised using matrix multiplication optimisation. The challenge of non-arithmetic operations is particularly addressed by the use of table lookups (tlookup) installed for each term of the K-digit base-b transformation in the product. These lookups facilitate the handling of operations like the $Softmax$ function efficiently. In this way, the distributed collaboration of $Verifiers$ are able to verify the $Servers$ in an efficient manner.

%%%%%%%%%%%%%%%%%%%% Related Works %%%%%%%%%%%%%%%%%%%%
\section{Related Works}
\label{sec:RelatedWorks}

This section reviews key developments in frameworks comparable to {\itshape eFedLLM}, focusing on their architectural and operational contributions to the field of LLMs. {\itshape opML} (Optimistic Machine Learning) \cite{conway2024opML} offers a blockchain-based solution to execute LLMs efficiently on standard PCs, broadening accessibility without relying on GPUs. Its core includes a Fraud Proof Virtual Machine (FPVM) to ensure computational integrity, a versatile ML Engine, and an Interactive Dispute Game for error resolution, significantly lowering costs and enhancing computational trustworthiness. However, the reliance on blockchain can introduce latency and scalability issues, potentially limiting its practical application in real-time scenarios.  {\itshape zkLLM} framework \cite{sun2024zkLLM} leverages specialized Zero-Knowledge Proofs (ZKPs) to secure the confidentiality and integrity of LLMs during inference processes, ensuring no sensitive data or model parameters are disclosed. It adapts ZKPs to complex operations such as the Attention mechanisms in LLMs, making model outputs trustworthy without revealing proprietary training data or techniques. Nevertheless, the computational overhead and complexity of ZKPs can be substantial, possibly hindering its efficiency and wider adoption. {\itshape PETALS} \cite{borzunov2022PETALS} aims to democratize LLM utilization through a collaborative platform that divides model layers across multiple servers, allowing for efficient resource use and latency reduction in model operations. It also supports model fine-tuning and sharing via a model hub, fostering community collaboration and resource-sharing among researchers and practitioners. However, its decentralized nature could complicate the consistency and synchronization of model updates, posing challenges in maintaining model accuracy and state across different servers. Furthermore, the incentive method in \cite{ding2022investigation} aligns with the approach in this paper, which includes smart contracts for automatically execute the incentive mechanisms in the distributed training process. \cite{ding2022survey} surveys the integration of blockchain technology with AI to mitigate issues like data misuse and enhance the robustness of AI models against malicious data or training contributions.

%%%%%%%%%%%%%%%%%%%% Conclusions %%%%%%%%%%%%%%%%%%%%
\section{Conclusions}
\label{sec:Conclusions}

This paper demonstrates an effective approach for enhancing the operational efficiency and accessibility of LLMs through the implementation of a transformer-based FL framework with model-parallel distributed training. By efficiently distributing computational and memory demands across the network, the proposed {\itshape eFedLLM} enables a broader range of users, particularly those with limited resources, to collaboratively train and utilize state-of-the-art LLMs. The integration of an incentive mechanism within the FL framework ensures the integrity and reliability of the training process by rewarding positive contributions and deterring malicious activities. Furthermore, the application of memory hierarchy strategies and SVD on weight matrices has proven to significantly reduce resource consumption and enhance system performance. The outcomes of this research not only optimize resource utilization but also democratize access to advanced AI technologies, allowing a diverse user base to contribute to and benefit from the development of trustworthy and efficient AI systems.

\bibliographystyle{ACM-Reference-Format}
\bibliography{sample-base}


\end{document}